\documentclass{colt2017} 

\title[Empirical Risk Minimization of Strict Saddle Problems]{Fast Rates for Empirical Risk Minimization of Strict Saddle Problems}
\usepackage{times}
\newcommand{\epsstab}{\epsilon_{\textrm{stab}}}
\newcommand{\epsgen}{\epsilon_{\textrm{gen}}}
\newcommand{\tens}{{\otimes 4}}
\newcommand{\tO}{\tilde{O}}

\newcommand{\cX}{\mathcal{X}}
\newcommand{\cD}{\mathcal{D}}

\newcommand{\cW}{\mathcal{W}}

\newcommand{\cZ}{\mathcal{Z}}

\newcommand{\cT}{\mathcal{T}}

\newcommand{\bE}{\mathbb{E}}
\newcommand{\reals}{\mathbb{R}}

\newcommand{\bN}{\mathbb{N}}

\newcommand{\poly}{\mathrm{poly}}

\newcommand{\tg}{\tilde{g}}
\newcommand{\tH}{\tilde{H}}

\newcommand{\hw}{\hat{w}}

\newcommand{\hg}{\hat{g}}
\newcommand{\hH}{\hat{H}}

\newcommand{\hF}{\hat{F}}

\newcommand{\hL}{\hat{L}}

\newcommand{\inner}[1]{\langle #1 \rangle}

\renewcommand{\eqref}[1]{Equation~(\ref{#1})}

\newcommand{\secref}[1]{Section~\ref{#1}}

\newcommand{\appref}[1]{Appendix~\ref{#1}}
\newcommand{\thmref}[1]{Theorem~\ref{#1}}
\newcommand{\lemref}[1]{Lemma~\ref{#1}}
\newcommand{\defref}[1]{Definition~\ref{#1}}

\newcommand{\remref}[1]{Remark~\ref{#1}}

\DeclareMathOperator*{\argmin}{arg\,min}

  \coltauthor{\Name{Alon Gonen} \Email{alongnn@cs.huji.ac.il} \and
    \Name{Shai Shalev-Shwartz} \Email{shais@cs.huji.ac.il}\\
   \addr School of Computer Science and Engineering, The Hebrew
   University, Jerusalem, Israel}

\begin{document}

\maketitle

\begin{abstract}
We derive bounds on the sample complexity of empirical risk
minimization (ERM) in the context of minimizing non-convex risks that admit
the strict saddle property. Recent progress in non-convex optimization
has yielded efficient algorithms for minimizing such functions. Our
results imply that these efficient algorithms are statistically stable
and also generalize well. In particular, we derive fast rates which resemble the bounds
that are often attained in the strongly convex setting. We specify our
bounds to Principal Component Analysis and
Independent Component Analysis. Our results and techniques may pave
the way for statistical analyses of additional strict saddle problems.
\end{abstract}

\section{Introduction}
Stability analysis is a central tool in statistical learning theory
(\cite{bousquet2002stability}). Roughly speaking, a learning algorithm
is stable if a slight change in the input of the algorithm does not
change its output much. It was shown
(\cite{shalev2010learnability,mukherjee2006learning}) that stability
characterizes learnability, and in particular, stability is equivalent
to the estimation error of empirical risk minimization.

Stability analysis has been mostly carried out in the context of convex
risk minimization. More concretely, some form of strong convexity is often assumed
(e.g., exp-concavity in \cite{koren2015fast,
  gonen2016average}). The crux of the technique is to show that minima of two
similar strongly convex (and Lipschitz/smooth) functions must be close
(\cite [Section 13.3]{shalev2014understanding}). 

In this paper we address the non-convex setting while restricting our
attention to recently studied ``nice'' non-convex problems. Namely, we
will consider non-convex functions which satisfy the strict saddle
property  (a.k.a. ridable or $\cX$-functions, see
\cite{sun2015nonconvex}). Roughly speaking, a strict saddle function has no spurious local
minimum and its saddle points are strict, in the sense that second-order
information suffices for identifying a descent direction. We also
assume that the restriction of the function to a certain neighborhood
of each of its minima is strongly convex.

Many important non-convex problems such as Principal
Component Analysis (PCA), complete dictionary recovery (\cite{sun2015nonconvex}), tensor decomposition,
ICA (\cite{ge2015escaping}, \cite{anandkumar2016homotopy}) and matrix completion
(\cite{ge2016matrix, bhojanapalli2016global}) are strict-saddle. Furthermore, there exist efficient
empirical risk minimizers (ERM) for these problems (e.g., SGD and
Cubic Regularization, see \secref{sec:related}). 

\section{Our contribution}
We consider the problem of minimizing a risk of the form 
\begin{equation} \label{eq:risk}
F(w) = \bE_{z \sim \cD} [f(w,z)]
\end{equation}
where for every $z \in \cZ$, $f(\cdot,z)$ is a twice continuously differentiable \emph{loss} function
defined over the closed set $\cW \subseteq \reals^d$. Given an
i.i.d. sample $S=(z_1,\ldots, z_d) \sim \cD^n$, the output of an ERM
algorithm is\footnote{We always assume the existence of a minima.}
\begin{equation} \label{eq:erm}
\hw \in \argmin_{w \in \cW}\,\left \{\hF(w) = \frac{1}{n} \sum_{i=1}^n
  f_{z_i}(w) \right\}~,
\end{equation}
The sample complexity of ERM is the minimal size of a sample $S$ for which $\bE[F(\hw)] -
  \min_{w \in \cW}F(w^\star) \le \epsilon$.\footnote{Alternatively,
    given $\epsilon$ and $\delta \in (0,1)$, we ask for the minimal
    size of a sample $S$ for which $F(\hw) -
  \min_{w \in \cW}F(w^\star) \le \epsilon$ with probability at least $1-\delta$.} We make the following assumptions on the loss functions:\\

\noindent \textbf{(A1)} For each $z$, $f(\cdot,z)$ is $\rho$-Lipschitz.\\
 \textbf{(A2)} For each $z$, $f(\cdot,z)$ is twice continuously differentiable and
\[
(\forall w \in \cW)~(\forall i \in [d])~~~|\lambda_i(\nabla^2 f(w,z))| \le \beta_1~. 
\]
\textbf{(A3)} For each $z$, the Hessian of $f(\cdot,z)$ is $\beta_2$-Lipschitz.\\

While for each example of strict saddle objective listed above one may
construct a dedicated sample complexity analysis, the goal of this
paper is to provide a systematic unified approach, which emphasizes the
geometric structure of the objective.

We distinguish between two cases. First, we consider the case where
the empirical risk is strict saddle (with high probability) and prove
stability and sample complexity bounds that depend solely on the strict saddle parameters of
the empirical risk and the Lipschitz constants. In particular, the
bound is dimensionality independent.
\begin{theorem} \label{thm:mainIntro}
Let $\epsilon \in (0,1)$. Suppose that that the empirical risk is
$(\alpha, \gamma, \tau)$-strict saddle with high probability (see \secref{sec:strictSaddle}). Then the
sample complexity of every ERM hypothesis is at most $\max \left \{ \frac{\beta_1}{\gamma}, \frac{\rho}{\tau},
  \frac{2 \rho^2}{\alpha \epsilon} \right \}$.
\end{theorem}
In some applications it may be easier to prove that $F$ itself is
strict saddle. Under the additional assumption that $\cW$
is bounded, we are
able to prove the next theorem.
\begin{theorem} \label{thm:mainTrue}
Suppose that $F$ (\eqref{eq:risk}) is $(\alpha,\gamma,\tau)$-strict saddle. The sample
complexity is at most $\tilde{O}\left ( d \left( \frac{ \rho}{\tau^2}+ \frac{
    \beta_1}{\gamma^2} + \frac{ \beta_1}{\alpha \epsilon} \right) \right)$.\footnote{The $\tilde{O}$ notation hides polylogarithmic dependencies.}
\end{theorem}
\begin{remark} \label{rem:kfir}
The proof of this theorem actually reveals something stronger. Suppose
we do not require all local minima of $F$ to be optimally global and consider
the family of empirical risk local minimizers. The same upper bound on
the number of samples stated in \thmref{thm:mainTrue} also suffices
for ensuring that the value, $F(\hat{w})$, associated with the output of any
such algorithm is $\epsilon$-close to the value of some local
minimum of $F$.
\end{remark}

We note that our bounds scale with $1/\epsilon$. In the literature,
such bounds are often referred to as \emph{fast
rates}, because standard concentration bounds
typically scale with $1/\epsilon^2$ (e.g., standard VC-dimension bounds in the
agnostic setting (\cite[Theorem 6.8]{shalev2014understanding}).
\subsection{Applications}
\subsubsection{PCA}
In \secref{sec:PCA} we apply \thmref{thm:mainIntro} to a stochastic
formulation of Principal Component Analysis (PCA). Our goal is to approximately recover
the leading eigenvector of the correlation matrix $\bE[xx^\top]$,
where $x$ is drawn according to some unknown distribution $\cD$ with bounded
support. The standard measure of success is given by the non-convex objective $\min_{\|w\|=1} -w^\top
\bE[xx^\top] w$. It is known that the sample complexity of ERM for
this problem is $\Omega(1/\epsilon^2)$ (\cite{blanchard2007statistical, gonen2016subspace}). 

Better bounds can be achieved under eigengap assumptions: there exists a gap, denoted
$G_{1,2}$, between the two leading eigenvalues of $\bE[xx^\top]$. We can
use the matrix Bernstein inequality to show that given an
i.i.d. sample of size $n =\Omega(\log(d/\delta)/G_{1,2}^2)$, with
probability at least $1-\delta$, a gap of the same order also appears
in the empirical correlation matrix. We then show that if such a gap exists, then the empirical risk is strict-saddle, where the
parameters are inversely proportional to $G_{1,2}$. This allows us to
deduce a bound of order $1/(n \cdot G_{1,2})$ on the stability and the
generalization error. We summarize the above in the next theorem.
\begin{theorem} \label{thm:pcaIntro}
The sample complexity of PCA is
$\tilde{O}\left(\frac{1}{G_{1,2}^2}+\frac{1}{\epsilon \cdot G_{1,2}} \right)$. 
\end{theorem}
This bound is superior to the general
$\tO(1/\epsilon^2)$ bound if $\epsilon = o(G_{1,2})$. One can claim that establishing the strict-saddle parameters of the
empirical risk already requires statistical tools which usually
already yield generalization bounds. Indeed, in the above
example, one can use the matrix Bernstein inequality to show that
$\tilde{O}(1/\epsilon^2_2)$ examples suffice in order to ensure that the
expected distance between the true correlation matrix and the empirical
correlation matrix (in operator norm) is at most
$\epsilon$. It is
then straightforward to establish the standard $1/\epsilon^2$ bound on
the generalization error. However, here we rely on Bernstein inequality only in order
to ensure that the gap in $\bE[xx^\top]$ appears also in the empirical
correlation matrix. Consequently, we are able to prove a better
bound (in a wide regime). 

\subsubsection{ICA}
In \secref{sec:PCA} we apply \thmref{thm:mainTrue} to a stochastic
formulation of Independent Component Analysis (ICA). Let $A$ be an orthonormal linear
transformation. Suppose that $x$ is uniform on $\{\pm 1\}^d$ and
let $y=Ax$. Our goal is to recover the matrix $A$ using the
observations $y$. As was shown in \cite{ge2015escaping}, this problem can be reduced to tensor
decomposition. Moreover, the latter can be formulated as a
strict saddle objective of the form (\ref{eq:risk}), which can be efficiently minimized using
SGD. 
\begin{theorem}
The sample complexity of ICA as formulated above is $\tilde{O}
\left(\poly(d)+\frac{d^{5/2}}{\epsilon} \right)$.
\end{theorem}
This result is meaningful in the regime where $d$ is small and we are
interested in a high accuracy solution.

\subsection{Our approach}
As we discussed above, most of the literature on stability analysis
presumes some notion of strong convexity. Strict saddle objectives
resemble strongly convex functions in the following sense: it is provided that the
restriction of the objective to a small neighborhood around any local
minimum is strongly convex. However, there are several major
differences. First, as opposed to strongly convex functions, there may
exists several minima. More importantly, there are regions of the
domain where the function is non-convex. 

Our analysis essentially reduces to the strongly convex
setting by excluding the other scenarios listed in \defref{def:ge}. Namely, we provide bounds on how many examples are needed in order to ensure
that a minimizer corresponding to a slight change in the input must be
in a strongly convex region around a local minimum $w^\star$. There is one more
subtlety we need to tackle; we are not guaranteed that the minimizer
of the (unmodified) empirical risk coincides with $w^\star$. However, as we shall see, since we deal with average
stability and since all local minima are global, we may assume that
this is the case w.l.o.g.

\section{Preliminaries} \label{sec:pre}

\subsection{Stability and generalization error}
\begin{definition} \label{def:uniStability}
Let $(z_1,\ldots,z_n) \sim \cD^n$ and let $\hw$ be an ERM (see
\eqref{eq:erm}). For every $i \in [n]$, let $\hw_i
\in \argmin_w \frac{1}{n-1}\sum_{j \neq i} f_j(w)$ and let $\Delta_i =
f_i(\hw_i) - f_i(\hw)$.\footnote{We do not assume uniqueness. The definition applies to any arbitrary rule
  for picking minimizers.} We say that the ERM algorithm is on average
stable with stability rate $\epsstab: \bN \rightarrow \reals_{>0}$ if
\[
\Delta:=\bE \left [\frac{1}{n} \sum_{i=1}^n \Delta_i \right ] \le \epsstab(n)~.
\]
Here and in the sequel, the expectation is taken both over the randomness of the
algorithm and the draw of $(z_1,\ldots,z_n)$.
\end{definition}
For $(z_1,\ldots,z_n) \sim \cD^n$, we define the generalization error
of ERM by $\epsgen(n)=\bE[\hF(\hw)-F(\hw)]$. The next lemma relates
the stability rate to the generalization error (see
\cite[Theorem 13.2]{shalev2014understanding}).
\begin{lemma} \label{lem:stabilityToGeneralization}
For every $n$,
\[
\bE_{S \sim \cD^{n-1}}[L(\hw)-L(w^\star)] \le \bE_{S \sim \cD^n} [\Delta(S)]
\] 
Therefore, for every $n$, $\epsgen(n) = \epsstab(n)$.
\end{lemma}
\subsection{Strict saddle functions} \label{sec:strictSaddle}
Due to their similarity to local extrema, saddle points raise a
fundamental challenge to optimization algorithms. Intuitively, the easier saddle
points are those for which second-order inrormation reveals a clear
descent direction. The following definition due to
\cite{sun2015nonconvex, ge2015escaping} captures this idea.
\begin{definition} \label{def:ge}
A twice continuously differentiable function $\hF: \reals^d
\rightarrow \reals$ is called $(\alpha, \gamma, \tau)$-strict saddle,
if it has no spurious local minimum, and for any point $x \in \reals^d$ at least one of the following conditions holds:
\begin{enumerate}
\item
$\|\nabla \hF(w))\| \ge \tau$
\item
$\lambda_{\min} (\nabla^2 \hF(w)) \le -\gamma$
\item
There exists $\nu>0$ and a local minimum $w^\star$ with $\|w-w^\star\| \le \nu$,
such that the restriction of $\hF$ to $2\nu$-neighborhood of $w^\star$ is
$\alpha$-strongly convex.\footnote{That is, for all $w$ in this
  neighborhood, $\nabla^2 F(w) \succeq \alpha I$}
\end{enumerate}
\end{definition}
\begin{remark} \label{rem:approx}
The requirement that every local minimum is globally optimal can be
relaxed. Namely, for a desired accuracy $\epsilon>0$, we may require
that every local minimum is $\epsilon/2$-optimal. Extending our
analysis to handle this case is straightforward.
\end{remark}
While \cite{ge2015escaping, sun2015nonconvex} also require a lower
bound on the magnitude of $\nu$ (which appears in the last condition), it turns out that
this quantity does not play any role in our analysis.

\section{Stability Bounds for Strict Saddle Empirical Risks: Unconstrained Setting}
In this section we consider the unconstrained setting (i.e., $\cW =
\reals^d$). Our main result (\thmref{thm:mainIntro}) follows from the
following theorem.
\begin{theorem} \label{thm:mainUnconstrained}
Let $\delta \in (0,1)$. Suppose that that the empirical risk $\hF$ is $(\alpha, \gamma, \tau)$-strict saddle (\defref{def:ge}) with probability at least $1-\delta$. If $n >  \max \left \{\frac{\rho}{\tau} ,
  \frac{\beta_1}{\gamma} \right\}$, then with
probability at least $1-\delta$, the expected generalization error and
stability rate of ERM are bounded by 
\[
\epsgen(n) = \epsstab(n) \le   \frac{2 \rho^2}{\alpha n}~.
\]
\end{theorem}
The proof reduces to the strongly convex case by bounding the number of
examples that are needed in order to exclude the first two scenarios
listed in \defref{def:ge}. Throughout the rest of this section we
assume that $\hF$ is $(\alpha, \gamma, \tau)$-strict saddle.
\begin{lemma} \label{lem:smallGrad}
Let $n > \rho/\tau$ and $(z_1,\ldots,z_n) \in \cZ^n$. Then for any $i \in [n]$, $\|\nabla \hF(\hw_i)\| \le \tau$.
\end{lemma}
\begin{proof}
Since $\hw_i$ minimizes $\frac{1}{n} \sum_{j \neq i} f_j(w)$, we have that 
\[
\hg_{-i}:=\frac{1}{n} \sum_{j \neq i} \nabla f_j(\hw_i) = 0~.
\]
Therefore, using the triangle inequality and the Lipschitzness of each
$f_i$, we obtain
\[
\|\nabla \hF(\hw_i)\| \le \|\hg_{-i} \|  + \frac{1}{n} \|  \nabla
f_i(\hw_i)\|  \le 0+ \rho/n < \tau~.
\]
\end{proof}
The proof of the next lemma has the same flavor.
\begin{lemma} \label{lem:nonSaddle}
Let $n > \beta_1/\gamma$ and $(z_1,\ldots,z_n) \in \cZ^n$. Then for any $i \in [n]$, $\lambda_{\min} (\nabla^2 \hF(\hw_i)) > - \gamma$.
\end{lemma}
\begin{proof}
By second-order conditions, $\hH_{-i}:=\frac{1}{n} \sum_{j \neq i}
\nabla^2  f_j(\hw_i)$ is positive semidefinite. Therefore, for all nonzero $v \in \reals^d$
\begin{align*}
\frac{v^\top \nabla^2 \hF(\hw_i) v}{v^\top v} =  \frac{v^\top \hH_{-i}
  v}{v^\top v} + \frac{1}{n} \frac{v^\top \nabla^2 f_i(\hw_i)
  v}{v^\top v} \ge 0 - \beta_1/n > -\gamma~.
\end{align*}
\end{proof}
It follows that for $n > \max\{\rho/\tau,\beta_1/\gamma\}$, we only need
to consider the third scenario listed in \defref{def:ge}.  
\begin{lemma} \label{lem:unconstStronglyConv}
For $n > \max\{\rho/\tau,\beta_1/\gamma\}$. Then,
\[
\epsgen(n) = \epsstab(n) = \frac{2 \rho^2}{\alpha n}~.
\]
\end{lemma}
\begin{proof} Let $(z_1,\ldots, z_n) \in \cZ^n$ for $n >
\max\{\rho/\tau,\beta_1/\gamma\}$ and fix some $i \in [n]$. According to
the previous two lemmas, $\hw_i$ lies in a neighborhood
around a local minimum $\bar{w}$ such that the restriction of $\hF$ to this
neighborhood is strongly convex. The crucial part is that since all
the local minima are global, for the sake of upper bounding the
stability we may assume w.l.o.g. that $\hw=\bar{w}$. Indeed, the
stability looks at the empirical risk of $\hw$, which is equal to the
empirical risk of $\bar{w}$ (here we can also allow an approximation
error of order $\epsilon$, see \remref{rem:approx}). From here the proof follows along
the lines of the standard proof in the Lipschitz and strongly
convex case (e.g., see \cite[Lemma 3]{gonen2016average}). We provide the details for completeness.

Fix some $i \in [n]$. By elementary properties of strongly convex functions, we have
\[
\hF(\hw_i) - \hF(\hw)  \ge \frac{\alpha}{2} \|\hw_i-\hw\|^2
\]
On the other hand, since $\hw_i$ minimizes the loss $w \in \cW \mapsto
\frac{1}{n} \sum_{j \neq i}  f_j(w)$, the suboptimality of $\hw_i$
w.r.t. the objective $\hF$ is
controlled by its suboptimality w.r.t. $f_i$, i.e.
\[
\hF(\hw_i) - \hF(\hw) \le \frac{1}{n}\Delta_i  
\]
Using Lipschitzness of $f_i$, we have
\[
\Delta_i \le \rho \|\hw_i-\hw\|
\]
Combining the above, we obtain
\[
\Delta_i^2 \le \rho^2 \|\hw_i-\hw\|^2 \le \frac{2\rho^2}{\alpha}
(\hF(\hw_i) - \hF(\hw)) \le \frac{2\rho^2}{\alpha n} \Delta_i 
\]
Dividing by $\Delta_i$ (we can assume w.l.o.g. that $\Delta_i > 0$) we
conclude the proof.
\end{proof}
This concludes the proof of \thmref{thm:mainUnconstrained}.


\section{Stability Bounds for Strict Saddle Empirical Risks: Constrained
  Setting}
We now consider the case where $\cW$ is described using equality constraints:
\[
\cW = \{w \in \reals^d: c_i(w)=0,~i=1,\ldots, m\}~,
\]
where for each $i$, $c_i(w)$ is twice continuously differentiable. 
\subsection{First and second-order conditions}
In this part we recall basic facts on first and second-order
conditions in the constrained setting (see for example \cite{borwein2010convex}). We introduce the Lagrangian $\hL: \reals^d \times \reals^m \rightarrow \reals$:
\[
\hL(w,\lambda) = \hF(w) + \sum_{i=1}^m \lambda_i c_i(w)~.
\]
We call a vector $\lambda \in \reals^m$ a Lagrange multiplier for $w \in \cW$ if $w$ is a critical point of $\hL(\cdot,\lambda)$. A vector $w \in \cW$ satisfies the linear independence constraint qualification (LICQ) condition if the set $\{\nabla c_i(w): i \in [m]\}$ is linearly independent.
\begin{theorem} \textbf{(KKT conditions)} \label{thm:kkt}
If $w \in \cW$ is a local minimum of $\hF$ and LICQ holds at $w$, then there exists a Lagrange multiplier $\lambda$ for $w$. 
\end{theorem}
Note that $\lambda$ can be found analytically using
\[
\lambda(w) = -(C(w))^\dagger \nabla \hF(w)~,
\] 
where $C$ is the matrix whose columns are $\nabla c_1(w), \ldots, \nabla c_m(w)$. In the sequel we often use the notation 
$$
\hL(w) = \hL(w,\lambda(w)),~\nabla \hL(w) = \nabla_w \hL(w,\lambda(w)),~\nabla^2 \hL(w) = \nabla _{ww} ^2 \hL(w,\lambda(w))~.
$$ 
The tangent space at any point $w \in \cW$ is defined by $\cT(w) = \{v \in \reals^d:\,(\forall i \in [m])~ v^\top \nabla c_i(w)=0 \}$. Following this notation, we observe that $\nabla \hL(w)$ is simply the projection of $\nabla \hL(w)$ onto the tangent space $\cT(w)$. In particular, \thmref{thm:kkt} provides conditions under which this projection vanishes. The next theorem extends the standard second-order conditions to our setting. 
\begin{theorem} \label{thm:secConst} \textbf{(Second-order necessary conditions)}
If $w \in \cW$ is a local minimum of $\hL$ and the set $\{\nabla c_i(w): i \in [m]\}$ is linearly independent, then for all $v \in \cT(w)$,
\[
v^\top \nabla^2 \hL(w) v \ge 0~.
\]
\end{theorem}



\subsection{Strict saddle property in the constrained setting}
We now provide a definition of the strict saddle property in the
constrained setting.
\begin{definition} \label{def:geConst}
A twice continuously differentiable function $\hF: \cW
\rightarrow \reals$ with constrains $c_i(w)$ and associated Lagrangian
$L$ is called $(\alpha,
\gamma, \tau)$-strict saddle if it has no spurious local
minimum, and for any point $w \in \cW$ at least one of the following conditions holds:
\begin{enumerate}
\item
$\|\nabla \hL(w)\| \ge \tau$
\item
There exists a unit vector $v \in \cT(w)$ s.t. $v^\top \nabla^2 L (w) v \le -\gamma$
\item
There exists a local minimum $w^\star$ such that 
$$
\frac{\|\nabla L(w)\|^2 }{2 \alpha} \ge \hL(w) - \hL(w^\star) \ge \frac{\alpha}{2} \|w-w^\star\|^2
$$
\end{enumerate}
\end{definition}
While our last condition is slightly different from its counterparts in
\cite{ge2015escaping, sun2015nonconvex}, we argue that it is often easier to
establish the condition stated here (e.g., see
\appref{app:tensor}).\footnote{Actually, it seems that our condition
  is also required in the proof of \cite{ge2015escaping}[Lemma 34]
  (see equation 121).} 

\subsection{Analysis in the constrained setting}
Throughout the section we prove that \thmref{thm:mainIntro} holds also
in the constrained setting. We assume that $\cW$ is described using $m$
equality constraints of the form $c_i(w)=0$ and that the LICQ holds
for all $w \in \cW$. 

As in the constrained setting, we first bound
the number of examples that are needed in order to exclude the two
first scenarios listed in \defref{def:geConst}. 
\begin{lemma} \label{lem:smallGradConst}
Let $n > \rho/\tau$ and $(z_1,\ldots,z_m) \in \cZ^n$. Then for any $i \in [n]$, $\|\nabla \hL(\hw_i)\| \le \tau$.
\end{lemma}
\begin{proof}
Since $\hw_i$ minimizes the risk w.r.t. $\frac{1}{n} \sum_{j \neq i}
f_j(w)$, we have that 
\[
\tg_{-i} = \frac{1}{n} \sum_{j \neq i} \nabla f_j(\hw_i) -
\sum_{s=1}^m \lambda_s(\hw_i) \nabla c_s(w)= 0~.
\]
Therefore, using the triangle inequality, we obtain
\[
\|\nabla \hL(\hw_i)\| \le \|\tg_{-i}\| +\frac{1}{n} \| \nabla
f_i(\hw_i)\|   \le \rho/n < \tau~.
\]
\end{proof}
\begin{lemma} \label{lem:nonSaddleConst}
Let $n > \beta_1/\gamma$ and $(z_1,\ldots,z_m) \in \cZ^n$. Then for any $i \in [n]$ and $v \in \cT(\hw_i)$  $v^\top \nabla^2 (\hL(\hw_i)) v \ge - \gamma$.
\end{lemma}
\begin{proof}
By second-order conditions, when restricted to $\cT(\hw_i)$, $\tH_{-i}:=\frac{1}{n} \sum_{j \neq i} \nabla^2  f_j(\hw_i)+\sum_{s=1}^m \lambda_s(w) \nabla^2 c_s(w)$ is positive semidefinite. Therefore, for every (nonzero) $v \in \cT(\hw_i)$, 
\begin{align*}
\frac{v^\top \nabla^2 \hL(\hw_i) v}{v^\top v}  \ge  \frac{v^\top
  \tH_{-i}  v}{v^\top v} + \frac{1}{n}  \frac{v^\top \nabla^2
  f_i(\hw_i)  v}{v^\top v}   \ge 0 - \beta_1/n > -\gamma~.
\end{align*}
\end{proof}
It follows that for $n > \max\{\rho/\tau,\beta_1/\gamma\}$, we only
need to consider the third scenario listed in
\defref{def:geConst}. The proof of the next lemma is almost identical to the
proof of \lemref{lem:unconstStronglyConv} and is therefore given in
the appendix (\appref{sec:omitted}). 
\begin{lemma} \label{lem:nonSaddleConst}
For $n > \max\{\rho/\tau,\beta_1/\gamma\}$ we have:
\[
\epsgen(n) = \epsstab(n) \le \frac{2 \rho^2}{\alpha n}~.
\]
\end{lemma}

\section{Application to PCA} \label{sec:PCA}
Consider the following stochastic formulation of PCA. Let $\cD$ be a
distribution over $\cZ \subseteq \reals^d$. We are interested in minimizing the objective
\[
F(w) =\frac{1}{2} \bE_{z \sim \cD} [\|z- ww^\top z\|^2]
\]
over all possible unit vectors $w \in \reals^d$.  We assume for
simplicity that $\cZ$ is contained in the Euclidean unit ball. It is well known that
the minimum is the leading eigenvector of the positive definite
matrix $\bE[zz^\top]$. As we shall see, this problem becomes strict saddle once we make the following standard assumption:\\

\noindent \textbf{(A4)} There is a positive gap, denoted $G_{1,2}$, between the two leading eigenvalues of $\bE[xx^\top]$.\\

\noindent Given a sample $(z_1,\ldots, z_n) \sim \cD^n$, let us denote
by $A= \frac{1}{n} \sum_{i=1}^n z_i z_i^\top$. The empirical risk is given by
\begin{align*}
\bar{F}(w) &= \frac{1}{2n} \sum_{i=1}^n \|z_i- ww^\top z_i\|^2
\end{align*}
One can easily see that an equivalent objective is given by
\[
\hF(w) = -\frac{1}{2} w^\top A w ~.
\]
Hence, the empirical risk admits exactly two (local and global)
minima, namely $u$ and $-u$,
where $u$ is the leading eigenvector
of $A$. 

We now would like to show that for sufficiently large $n$, the empirical
risk is strict saddle. The first step should be to translate our eigengap
assumption on $\bE[zz^\top]$ to a similar assumption on $A$. The following lemma, which follows from a simple application of the Matrix Bernstein
inequality (\cite{tropp2015introduction}[Section 1.6.3]), shows that for sufficiently large $n$, the eigengap between the
two leading eigenvalues of $A$ is $\Omega(G_{1,2})$. 
\begin{lemma} \label{lem:concentrationGap}
Let $\delta \in (0,1)$. For $n = \Omega
\left(\frac{\log(d/\delta)}{G_{1,2}^2} \right)$, we have that with probability at
least $1-\delta$,
\[
\|A-\bE[xx^\top]\| \le G_{1,2}/2 = : G
\]
It follows that with probability at least $1-\delta$, the gap between the leading eigenvalues of $A$ is at
least $G$.
\end{lemma}
The following theorem implies \thmref{thm:pcaIntro}.
\begin{theorem} \label{thm:pcaSaddle}
For any $\delta \in (0,1)$, if the sample size $n$ is $\Omega
\left(\frac{\log(d/\delta)}{G_{1,2}^2} \right)$, then with probability
at least $1-\delta$, the PCA objective
satisfies the conditions in \defref{def:geConst} with
$\tau,\gamma,\alpha \in \Omega(G_{1,2})$. Consequently, for any $n=\Omega
\left(\frac{\log(d/\delta)}{G_{1,2}^2} \right)$,
\[
\epsgen(n) = \epsstab(n) \le \frac{4 }{n \cdot G _{1,2} }~.
\]
\end{theorem}
\begin{proof} \textbf{(idea)}
Critical points of the Lagrangian correspond to eigenvectors of
$A$ (where we refer to the zero vector as an eigenvector as well). We show that if the gradient at some point $w$ is small, then $w$
either belongs to a strongly convex region around the leading
eigenvector or to a strict saddle neighborhood of another eigenvector
(or $0$).
\end{proof}
The proof is given in \appref{app:pca}.

\section{Sample Complexity Bounds for Strict Saddle Expected Risks}
In some cases it may be easier to establish the strict saddle property
of the expected risk (\eqref{eq:risk}). We now assume that $F$ is
$(\alpha,\tau,\gamma)$-strict saddle. We consider the constrained setting
and denote the Lagrangian of $F$ by $L$. We add the following
boundedness assumption:\\

\noindent \textbf{(A4)} The set $\cW$ is contained in $\{w: \|w\| \le
B\}$.\\

The proof of \thmref{thm:mainTrue} is given in \appref{sec:omitted}. Below we give the main idea.
\begin{proof} \textbf{(idea) of \thmref{thm:mainTrue}}
We use Matrix Bernstein inequality together with covering to show that with high probability, points with large
gradient do not form minima of $\hF$. Similar argument shows that strict saddle points of $F$ do not become
minima of $\hL$. Then, we can restrict ourselves
to strongly convex regions of $F$ and show that any $w$ with
$F(w)-\min_{w' \in \cW} F(w') > \epsilon$ can not be a minimum of
$\hF$.
\end{proof}



\section{Application to ICA Through Tensor Decomposition} 
A $p$-order tensor is a $p$-dimensional array. Here we focus on
$4$-order tensors. For a tensor $T \in \reals^{d^4}$ and indices
$i_1,\ldots,i_4 \in [d]$, we denote the $(i_1,\ldots,i_4)$-th entry of
$T$ by $T_{i_1,\ldots,i_4}$. Every $d$-dimensional vector $a$ induces
a rank-one $4$-order tensor, denoted $a^{\otimes 4}$, where
$a^{\otimes 4}_{i_1,i_2,i_3,i_4}$ is $a_{i_1} a_{i_2} a_{i_3}
a_{i_4}$. We can present the tensor $T$ using a multilinear form. Given
vectors $u,v,z,w \in \reals^d$, we define
\[
T(u,v,z,w) = \sum_{i_1,i_2,i_3,i_4} T_{i_1,\ldots,i_4} u_{i_1} v_{i_2}
z_{i_3} w_{i_4}
\]
The tensor $T$ has an orthogonal decomposition if it can be written as
\begin{equation} \label{eq:tensorDec}
T = \sum_{i=1}^d a_i^\tens~.
\end{equation}
In case that such decomposition exists, it is unique up to a permutation of the $a_i$'s and sign flips. 
A central problem in machine learning is to compute the tensor
decomposition of a given tensor $T$
(\cite{anandkumar2014tensor}). While we have exponentially many
equivalent solutions, the average of two solutions does not form a solution. Hence, any
reasonable formulation of this problem must be non-convex. Luckily, as was shown in \cite{ge2015escaping}, there exists a strict saddle formulation of this
problem. 

For simplicity, we consider the problem of finding one component (one
can proceed and find all the components using deflation). Consider
the following objective:
\begin{equation} \label{eq:geTensor}
\max_{\|u\| =1}  T(u,u,u,u)~.
\end{equation}
\begin{lemma} \label{lem:geTensor} \textbf{(\cite{ge2015escaping})}
Suppose that $T$ admits a Tensor decomposition as in
(\ref{eq:tensorDec}). The only local minima of (\ref{eq:geTensor}) are
$\pm a_i$. Furthermore, the objective (\ref{eq:geTensor}) is $(\alpha, \gamma, \tau)$-strict
saddle with $\alpha=\Omega(1), \gamma = 7/d$ and $\tau=1/\poly(d)$. Last,
for $p=1,2,3$, the magnitude of the $p$-th order derivative of this objective is $O(\sqrt{d})$.
\end{lemma}
Although our definition of strict saddle functions in the constrained
setting is slightly different from its counterpart in
\cite{ge2015escaping}, it is not hard to show that
\lemref{lem:geTensor} still holds (see \appref{app:tensor}).

In applications, we often have access to $T$ only through a
stochastic oracle. Following \cite{ge2015escaping}, we consider the
following formulation of ICA. Let $A$ be an orthonormal linear
transformation. Suppose that $x$ is uniform on $\{\pm 1\}^d$ and
denote by $y=Ax$. Our goal is to recover the matrix $A$ using the
observations $y$. It turns out that ICA reduces to tensor
decomposition. Namely, define $Z \in \reals^{d^4}$ by
\[
(\forall i \in [d])~~Z(i,i,i,i) = 3,~~(\forall i \neq j)~~Z(i,i,j,j)=
Z(i,j,j,i) = Z(i,j,i,j)=1~,
\]
where all other entries of $Z$ are zero. 
\begin{lemma}
The expectation $\frac{1}{2} \bE[Z-y^\tens]$ is equal to $T$, where the
vectors participating in the decomposition of $T$ correspond to
columns of $A$.
\end{lemma}
Following the lemma, we can rewrite (\ref{eq:geTensor}) as the
following expected risk:
\begin{equation} \label{eq:geTensor2}
\max_{\|u\| =1} \bE\left[  \frac{1}{2} \left(Z-y^\tens \right) \right]  (u,u,u,u)~.
\end{equation}
Furthermore, as was shown in \cite{ge2015escaping}, one can
efficiently compute a stochastic gradient and use SGD to optimize this
objective. Using \lemref{lem:geTensor} and \thmref{thm:mainTrue}, we
conclude that the sample complexity of extracting a single column of $A$ is $\tilde{O}
\left(\poly(d)+\frac{d^{3/2}}{\epsilon} \right)$. The sample
complexity of extracting all the columns is $\tilde{O}
\left(\poly(d)+\frac{d^{5/2}}{\epsilon} \right)$.

\section{Related Work} \label{sec:related}
\subsection{Efficient ERM for Strict Saddle Functions}
There is a growing interest in developing efficient algorithms for
minimization of strict saddle functions. We mention two central
approaches. Intuitively, one can escape from a saddle point by moving in the
direction of the eigenvector corresponding to the minimal
eigenvalue. This intuition has been made precise by Nesterov and Polyak
(\cite{nesterov2006cubic}). More surprisingly, in \cite{ge2015escaping} it was shown that a variant of SGD also converges to a local
minimum. Recent improvements in terms of runtime are given in
\cite{agarwal2016finding, levy2016power}. 
\subsection{Stability of SGD}
Recently, \cite{hardt2015train} analyzed the stability of the SGD algorithm both in a convex and
non-convex setting. As we mentioned above, in our setting, SGD forms
an empirical risk minimizer. Our bounds on the stability rate of SGD
in this setting improve over the (more general) bounds of
\cite{hardt2015train}. In particular, our bounds imply that SGD can be
trained for arbitrarily long time. 

\subsection{Generalization Bounds using SGD}
It is known that one can obtain generalization bounds directly using
SGD (\cite{shalev2014understanding}[Chapter 14]). Hence, the time
complexity bound of \cite{ge2015escaping} translates into identical
sample complexity bound. However, their bounds, which scale with
$1/\epsilon^4$, are inferior to our bounds when high accuracy is desired.
\subsection{Fast rates for PCA}
Generalization bounds for stochastic PCA have been studied in
\cite{bousquet2002stability, gonen2016subspace}. Both works prove an
upper bound of $1/\sqrt{n}$ on the generalization error in the general
case. The latter work (which also considers the challenge of partial
information) establishes a matching lower bound. The former work also
considers the case of a positive eigengap between the leading
eigenvalues of $\bE[xx^\top]$\footnote{More generally, these works
  consider the task of approximating the $k$ leading eigenvectors. It
  is not hard to extend our results to this task as well.} and
establishes fast rates similar to our bounds using Local Rademacher
complexities. We believe that these techniques are much more involved
than our techniques and lack any geometric interpretation.

\section*{Acknowledgments} 
We thank Kfir Levy for bringing \remref{rem:kfir} into our
attention. We also thank Nati Srebro for helpful discussions.

 \newpage
 \bibliography{bib}
 \newpage
 \appendix

\section{PCA Is Strict Saddle: Complete Proof}  \label{app:pca}
This section is devoted to the proof of \thmref{thm:pcaSaddle}. Let us
start with some basic calculations. The gradient and the Hessian of
$\hF(w)$ are given by   
$$
\nabla \hF(w) = -Aw,~~\nabla^2 \hF(w) = -A~.
$$
It is apparent that both the domain and the the objective are not
convex. The following lemma is immediate.
\begin{lemma}
The restriction of $\hF$ to the unit sphere in $\reals^d$ is $1$-Lipschitz and $1$-smooth.
\end{lemma}
Letting $c(w) = \frac{1}{2} (\|w\|^2-1)$, the Lagrangian is given by
\[
\hL(w,\lambda) = \hF(w) + \lambda c(w) = -\frac{1}{2} w^\top Aw + \frac{\lambda}{2} (\|w\|^2-1)~.
\]
It follows that 
\[
\lambda(w) = w^\top A w~.
\]
Therefore, the gradient and the Hessian of $\hL(w)$ are given by
\[
\nabla \hL(w) = (\lambda(w) I-A)w ,~~\nabla^2 \hL(w) = (\lambda(w)I-A)
\]
Note also that LICQ trivially holds at any point $w \in \cW$.
\begin{proof} \textbf{(of \thmref{thm:pcaSaddle})}
Let $w$ be a unit vector in $\reals^d$ and suppose that $\|\nabla f(w) \| \le \tau = c G$ for some constant $c \in (0,1/32)$. We show that $w$ satisfies either the second or the third condition in \defref{def:geConst}.\\

\noindent \textbf{First step (setup):} \\
Let $w = \sum_{i=1}^d \alpha_i u_i$ be the decomposition of $w$
according to the eigenbasis of $A$. Note that by the optimality of $u_1$, $\lambda \le \lambda_1$. Also, by assumption
\begin{align} \label{eq:gradSmallEigen}
\tau ^2 \ge \|(\lambda I -A)w\|^2 = w^\top \sum_{i=1}^d (\lambda-\lambda_i)^2 u_i u_i^\top w = \sum_{i=1}^d \alpha_i^2 (\lambda-\lambda_i)^2~.
\end{align}

\noindent \textbf{Second step (bounding the mass of distant eigenvalues):}\\
Note that $\|\alpha\|^2=1$, hence the vector $\alpha^2 = (\alpha_1^2,\ldots, \alpha_d^2)$ can be seen as a probability vector. We next apply Markov's inequality in order to bound the mass of eigenvalues located far from $\lambda$. For every $t =0,1,\ldots$, define
\[
I_t = \{i \in [d]: |\lambda-\lambda_i| \le 2^t \tau\}~.
\]
We claim that for every $t$,
\begin{align} \label{eq:markov}
\sum_{i \notin I_t } \alpha_i^2 \le 2^{-2t}~.
\end{align}
Indeed, for $t=0$ the bound is trivial and for $t \ge 1$ we apply (\ref{eq:gradSmallEigen}) to otbain
\begin{align*}
\tau^2 \ge \sum_{i \notin I_t} \alpha_i^2(\lambda-\lambda_i)^2 \ge 2^{2t} \tau^2 \sum_{i \notin I_t} \alpha_i^2~.
\end{align*}
By rearranging, we conclude the claim. \\

\noindent  \textbf{Third step (the strongly convex case):}\\
Consider the case where $1 \in I_4$. It follows that
\begin{align*}
\lambda_1-16 cG =\lambda_1 - 2^4 \tau \le \lambda = \sum_{i=1}^d \alpha_i^2 \lambda_i \le \alpha_1^2 \lambda_1+\sum_{i=2}^d \alpha_i^2(\lambda_1-G) = \lambda_1- G \sum_{i=2}^d \alpha_i^2~,
\end{align*}
where the last equality uses the fact that $\sum_{i=1} ^d \alpha_i^2
=1$. Hence, $\sum_{i=2}^d \alpha_i^2 \le 16 c$, so
\begin{equation} \label{eq:strongCase}
\alpha_1^2 \ge (1-16 c) \ge 1/2 \Rightarrow \sum_{i \ge 2} \alpha_i^2
\le 1/2~.
\end{equation}
We now show that $\hF(w)-\hF(u_1)\ge \frac{G}{4}\|w-u_1\|^2$. First we calculate the distance between $w$ and $u_1$:
\begin{equation} \label{eq:distW}
\|w-u_1\|^2 = (\alpha_1-1)^2 + \sum_{i \ge 2} \alpha_i^2 = \sum_{i=1}^d \alpha_i^2 + 1-2\alpha_1 = 2(1-\alpha_1)~.
\end{equation}
Since $w$ and $u_1$ are feasible, $\hF(w) = \hL(w)$ and $\hF(u_1)=\hL(u_1)$. Since $\hL$ is quadratic and $u_1$ is optimal (hence $\nabla \hL(u_1)=0$), we have 
\[
\hF(w) = \hF(u_1) + \inner{\nabla \hL(u_1),w-u_1} + \frac{1}{2}(w-u_1)^\top \nabla^2 \hL(u_1) (w-u_1) = \hL(u_1) +  \frac{1}{2}(w-u_1)^\top \nabla^2 \hL(u) (w-u_1)
\]
It is left to bound the quadratic term from below. Since $0 \le \lambda_1-\lambda \le 16 c G$ for $c \in (0,1/32)$,
\begin{equation} \label{eq:strongCase1}
\lambda_1 - \lambda \le G/2 \Rightarrow (\forall i \ge 2) ~~~\lambda-\lambda_i \ge G/2~.
\end{equation}
Therefore, 
\begin{align*}
\frac{1}{2}(w-u_1)^\top \nabla^2 \hL(u) (w-u_1) &= (\alpha_1-1)^2 (\lambda_1-\lambda_1) + \sum_{i\ge 2} \alpha_i^2(\lambda_1-\lambda_i)  \\
&\ge G \sum_{i \ge 2} \alpha_i^2  \ge G (-(\alpha_1-1)^2+\sum_{i \ge
  2} \alpha_i^2) \\
&=\frac{G}{2}(\sum_{i=1}^d \alpha_i^2 - 2 \alpha_1^2+2\alpha_1-1) =
  \frac{G}{2} (2 \alpha_1-2\alpha_1^2)  \\
& = \frac{G}{2}   2\alpha_1 ( 1-\alpha_1) \underbrace{=}_{(\ref{eq:distW})} \frac{G}{2}   \alpha_1 \|w-u_1\|^2
  \\
&\underbrace{\ge}_{(\ref{eq:strongCase})} \frac{G}{4} \|w-u_1\|^2~,
\end{align*}
We deduce that
\[
\hF(w) - \hF(u_1) \ge  \frac{G}{4} \|w-u_1\|^2~.
\]
On the other hand,
\begin{align*}
\frac{1}{2}(w-u_1)^\top \nabla^2 \hL(u) (w-u_1) &= \sum_{i \ge
                                                  2}\alpha_i^2
                                                  (\lambda-\lambda_i+\lambda_1-\lambda)
                                                  \underbrace{\le}_{\ref{eq:strongCase},
                                                  \ref{eq:strongCase1}}  \sum_{i \ge 2} \alpha_i^2 (\lambda-\lambda_i+\lambda_1-\lambda)\\&+\alpha_1^2(\lambda-\lambda_i)-\sum_{i\ge 2}\alpha_i^2(\lambda_1-\lambda)
= \sum_{i \ge 2} \alpha_i^2 (\lambda-\lambda_i) \\& \underbrace{\le}_{\ref{eq:strongCase1}} \sum_{i \ge 2}
    \alpha_i^2 (\lambda-\lambda_i)^2/(G/2)  \le \sum_{i \ge 1}
    \alpha_i^2 (\lambda-\lambda_i)^2/(G/2) \\&= \frac{\|\nabla L(w)\|^2}{2(G/4)}~.
\end{align*}

\noindent \textbf{Fourth step (the strict saddle case): }\\
Consider the case where $1 \notin  I_4$. We construct a vector $v \in
\cT(w)$ such that $\frac{v^\top \nabla^2 \hL(w) v}{\|v\|^2}$ is
proportional to $-G$. Let
\[
v = u_1-\alpha_1 w 
\]
Note that $v$ is perpendicular to $w$, hence $v \in \cT(w)$. Also note that
\[
v = (1-\alpha_1^2)u_1 - \alpha_1 \sum_{i \ge 2} \alpha_i u_i~
\]
Hence, 
\[
v^\top \nabla^2 \hL(w) v  = (1-\alpha_1^2) (\lambda-\lambda_1) + \sum_{i \ge 2} \alpha_i^2(\lambda-\lambda_i)~.
\]
We bound each of the terms in the RHS. Using (\ref{eq:markov}) we upper bound $\alpha_1^2$ by $2^{-8}$. Since $\lambda \le \lambda_1$, we have
\[
(1-\alpha_1^2) (\lambda-\lambda_1) \le - \frac{255}{256} \cdot 16 \tau \le -15 \tau~.
\]
On the other hand, denoting $J_t =  I_t \setminus \bigcup_{s=0}^{t-1} I_s$, we have
\begin{align*}
\sum_{j \ge 2} \alpha_j^2(\lambda-\lambda_j) &\le  \sum_{j \ge 1} \alpha_j^2 |\lambda-\lambda_j| = \sum_{t=0}^\infty \sum_{j \in J_t} \alpha_j^2 |\lambda-\lambda_i| \le  \sum_{t=0}^\infty \sum_{j \in J_t} \alpha_j^2 2^t \tau \\
& = \le  \tau \sum_{t=0}^\infty 2^{-2t} 2^t = 2 \tau~,
\end{align*}
where the last inequality follows from (\ref{eq:markov}). Note also that $\|v\| \le 2$. Overall, we obtain that
\[
\frac{v^\top \nabla^2 \hL(w) v}{\|v\|^2}  \le -13\tau /2  \le -6 c G~.
\]
\end{proof}

\section{ICA is Strict Saddle: Establishing Strong Convexity}  \label{app:tensor}
Our notion of strong convexity in \defref{def:geConst} is slightly
different from its counterpart in \cite{ge2015escaping}. We now show
that \lemref{lem:geTensor} holds using our definitions. 

Let $w \in \cW$. To simplify the presentation, we assume that $a_i =
e_i$ for all $i$ (alternatively, we could do a change of coordinates
to $w$, which does not affect the structure of the problem).
Denote
\[
\tau_0 = (10d)^{-4} ,~\tau = 4 \tau_0^2,~D = 2d
\tau_0,~I(w) = \{i \in [d]:\,|w_i| > \tau_0\}
\]
Suppose that $\|\nabla L(w)\| \le \tau$, where $L$ is the Lagrangian
associated with the expected risk $F$. It was shown in
\cite{ge2015escaping} that if $|I(w)| \ge 2$, then $w$ is a strict
saddle point. Hence, it is left to consider the case where $|I(w)|=1$. Assume w.l.o.g. that $I(w) = \{1\}$. 

\begin{lemma}
The suboptimality of $w$ w.t.t. the
minimum $e_1$ is bounded below by 
$$
F(w)-F(e_1) \ge \frac{1}{4} \|w-e_1\|^2~.
$$
\end{lemma}
\begin{proof}
Since $w$ is a unit vector,
\[
1 \ge w_1^2 = 1- \sum_{i \ge 2} w_i^2   \ge 1-d\tau_0^2
\]
The squared distance between $w$ and the local minimum $e_1$ is at most
\[
\|w-e_1\|^2 = (1-w_1)^2 + \sum_{i \ge 2} w_i^2  \le 2d \tau_0^2
\le D^2~.
\]
Let $c(w) = \frac{1}{2}(\|w\|^2-1)$. Since $c(w)=c(e_1)=0$, using the
$1$-smoothness of $c$ we obtain
\[
0=c(w) \le c(e_1) + \nabla c(e_1)^\top (w-e_1) + \frac{1}{2} \|w-e_1\|^2 = e_1(w-e_1)~.
\]
Hence, 
\begin{equation} \label{eq:tangentClose}
(1-w_1) ^2 =(e_1^\top (e_1-w))^2 \le \frac{1}{4} \|w-e_1\|^4 \le \frac{1}{4} \|w-e_1\|^2
\end{equation}
As \cite{ge2015escaping} show, The Hessian of $L$ at $e_1$ is a diagonal matrix with $4$ on the
diagonals except for the first diagonal entry whose value is
$-8$. Since $F(w) = L(w)$ and $F(e_1) = L(e_1)$,
\begin{align*}
F(w) =  F(w_1)+\underbrace{\nabla L(e_1)^\top}_{=0} (w-e_1) + \frac{1}{2}(w-e_1)^\top \nabla ^2 L(w') (w-e_1)
\end{align*}
for some $w'$ that lies on the line between $w$ and $e_1$. Note that
\begin{align*}
&\frac{1}{2}(w-e_1)^\top \nabla ^2 L(w') (w-e_1) \\ &=
\frac{1}{2}(w-e_1)^\top \nabla ^2 L(e_1) (w-e_1) +
\frac{1}{2}(w-e_1)^\top (\nabla ^2 L(w')-\nabla^2 L(e_1)) (w-e_1)~.
\end{align*}
Using (\ref{eq:tangentClose}), we bound the first term in the RHS by
\begin{align*}
(w-e_1)^\top \nabla ^2 L(e_1) (w-e_1) &= -8(1-w_1)^2 + 4\sum_{i \ge 2} w_i^2  = 4((1-w_1)^2 + \sum_{i \ge 2} w_i^2) - 12 (1-w_1)^2 \\&\ge 4 \|w-u_1\|^2  - 3\|w-u_1\|^2 = \|w-u_1\|^2
\end{align*}
Using the $O(\sqrt{d})$-Lipschitzness of the Hessian and the fact that
$\|w'-e_1\|\le D$, the second term is bounded by
\[
(w-e_1)^\top (\nabla ^2 L(w')-\nabla^2 L(e_1)) (w-e_1) \le
\|w-e_1\|^2 \|w'-e_1\| \sqrt{d} \le \frac{1}{2} \|w-e_1\|^2~.
\]
All in all, 
\[
F(w) - F(e_1) \ge \frac{1}{4} \|w-e_1\|^2~.
\]
\end{proof}
\begin{lemma}
The suboptimality of $w$ w.t.t. the
minimum $e_1$ is bounded above by 
$$
F(w)-F(e_1) \le O( \|\nabla L(w)\|^2)~.
$$
\end{lemma}
\begin{proof}
Using the previous lemma and the Lipschitzness of the Hessian, one can easily show that 
\[
F(e_1) \ge F(w) + \nabla L(w)^\top (e_1-w) + \frac{c}{2} \|e_1-w\|^2
\]
for some constant $c \in (0,1)$. The RHS is at most
\[
\min_{z \in \reals^d} F(w) + \nabla L(w)^\top (z-w) + \frac{c}{2} \|z-w\|^2
\]
The minimum is attained at $z=w-c^{-1}\nabla L(w)$. The desired
inequality follows by substitution.
\end{proof}

\section{Omitted Proofs} \label{sec:omitted}
\begin{proof} \textbf{(of \lemref{lem:nonSaddleConst})} 
According to the previous two lemmas, $\hw_i$ lies in neighborhood
around a local minimum $w^\star$ such that the restriction of $\hF$ to this
neighborhood is strongly convex. As in the unconstrained setting we may assume w.l.o.g. that $\hw=w^\star$. 

Fix some $i \in [n]$. By assumption
\[
\hF(\hw_i) - \hF(\hw) = \hL(\hw_i) - \hL(\hw)  \ge \frac{\alpha}{2} \|\hw_i-\hw\|^2
\]
On the other hand, since $\hw_i$ minimizes the loss $w \in \cW \mapsto
\frac{1}{n} \sum_{j \neq i}  f_j(w)$, the suboptimality of $\hw_i$
w.r.t. the objective $\hF$ is
controlled by its suboptimality w.r.t. $f_i$, i.e.
\[
\hF(\hw_i) - \hF(\hw) \le \frac{1}{n}\Delta_i  
\]
Using Lipschitzness of $f_i$, we have
\[
\Delta_i \le \rho \|\hw_i-\hw\|
\]
Combining the above, we obtain
\[
\Delta_i^2 \le \rho^2 \|\hw_i-\hw\|^2 \le \frac{2\rho^2}{\alpha}
(\hF(\hw_i) - \hF(\hw)) \le \frac{2\rho^2}{\alpha n} \Delta_i 
\]
Dividing by $\Delta_i$ (we can assume w.l.o.g. that $\Delta_i > 0$) we
conclude the proof.
\end{proof}

\begin{proof} \textbf{(of \lemref{lem:concentrationGap})}
The first part is a direct application of Bernstein
inequality (\cite{tropp2015introduction}[Section 1.6.3]). It is left
to prove that if $A,B$ are positive semidefinite and $\|A-B\| \le \epsilon$,
then for all $i$, $|\lambda_i(A)-\lambda_i(B)| \le \epsilon$. Indeed, 
\begin{align*}
\lambda_i(B) &=  \max_{\dim(V)=i} \min_{v \in V}
\frac{v^\top B v} {v^\top v}    \\
&  =  \max_{\dim(V)=i} \min_{v \in V}
\frac{v^\top A v+ v^\top (B-A)v } {v^\top v} \\
& \le  \max_{\dim(V)=i} \min_{v \in V}
\frac{v^\top A v} {v^\top v} + \max_{v \in V}
\frac{v^\top (B-A) v} {v^\top v}  \\
& = \lambda_i(A)+\epsilon~.
\end{align*}
Analogous proof shows that $\lambda_i(A) \le \lambda_i(B)+\epsilon$.
\end{proof}

\begin{proof} \textbf{(of \thmref{thm:mainTrue})}
Recall that the Lagrangian of $\hF$ is denoted by $\hL$. We first show
that with high probability, points with large
gradient do not form minima of $\hF$. Similar argument shows that strict saddle points of $L$ do not become
minima of $\hF$. Then, we can restrict ourselves
to strongly convex regions of $L$ and show that any $w$ with
$F(w)-\min_{w' \in \cW} F(w') > \epsilon$ can not be a minimum of
$\hF$.

Fix some point $w \in \cW$ with $\|\nabla L(w)\| \ge \tau$. Using
matrix Bernstein inequality, we deduce that if $n =\Omega(\rho
\log(d/\delta)/\tau^2))$, then $\|\nabla \hL(w)\| \ge \tau/2$. 
Also, using Property A2, we have that for any $u \in \cW$ with
$\|u-w\| \le r_1:= \min\{\frac{ \tau}{4\beta_1},1\}$, $\|\nabla
\hL(u)\| \ge \tau/4$. Since $\cW$ is bounded we can cover $\cW$ using $(4B/r_1)^d$ balls of radius $r_1$ (for example, see
the proof of \cite{matouvsek2002lectures}[Lemma 13.11.1]). By applying
the union bound we deduce that if $n = \Omega(d\rho
\log(dB/(r_1\delta)/\tau^2)$, then with probability at least $1-\delta$,
all points $w$ with $\|\nabla L(w)\|
\ge \tau$ satisfy $\|\nabla \hat{L}(w)\| \ge \tau/4$. 

We next fix some point $w \in \cW$ for which there exists a unit
vector $v \in \cT(w)$ with $v^\top (\nabla^2
L(w)) v \le -\gamma$. Using
matrix Bernstein inequality, we deduce that if $n =\Omega(\beta_1
\log(d/\delta)/\gamma^2)$, then $v^\top \nabla^2
L(w) v \le -\gamma/2$. 
Also, using Property A3, we have that for any $u \in \cW$ with
$\|u-w\| \le r_2:= \min\{\frac{ \gamma}{4\beta_2},1\}$, there exists
$v \in \cT(u)$ with $v^\top \nabla^2 L(u) v$. Since $\cW$ is
bounded, we can cover $\cW$ using $(4B/r_2)^d$ balls of radius $r_2$. By applying
the union bound, we obtain that a sample of size $n = \Omega(d \beta_1
\log(dB/(r_2\delta))/\gamma^2)$ ensures that with probability at least
$1-\delta$, $\gamma$-strict saddle points of F are $\gamma/2$-strict
saddle of $\hF$.

In particular, using \thmref{thm:kkt} and \thmref{thm:secConst} we
deduce that strict saddle points of $F$ and points with large gradient
do not form local minima of $\hat{L}$.

Consider now vectors $w \in \cW$ which belong to a strongly convex
region around some minimum of $F$, denoted $w^\star$. Suppose that
$F(w)-F(w^\star) > \epsilon$. By strong convexity, $\|\nabla L(w)\|^2 \ge 2 \alpha
\epsilon$. Using concentration and covering as above, we conclude that for
$n = \Omega(\beta_1 \log(dB/(r_1\delta))/(\alpha \epsilon))$, then with
probability at least $1-\delta$, $\|\nabla
\hL(w)\|^2 \ge \alpha \epsilon$, hence $w$ is not a local minimum
of $\hF$.
\end{proof}

\end{document}